\documentclass[twocolumn,10pt]{article}

\usepackage[round]{natbib}

\usepackage[latin1]{inputenc}
\usepackage[inline]{enumitem}
\usepackage{mystyle,amssymb,dsfont,amsmath,amsthm,mathtools}
\usepackage{rotating} 
\usepackage{url}

\newlength\titlebox 
\begin{document}

\title{Stochastic Deep Networks}

\author{ 
\begin{tabular}{c}
Gwendoline de Bie \\ DMA, ENS \\ {\small \url{gwendoline.de.bie@ens.fr}}
\end{tabular}
\quad 
\begin{tabular}{c}
Gabriel Peyr\'e \\ CNRS \\ DMA, ENS \\ {\small \url{gabriel.peyre@ens.fr}}
\end{tabular}
\quad 
\begin{tabular}{c}
Marco Cuturi \\ Google Brain \\ CREST ENSAE\\ {\small \url{cuturi@google.com}}
\end{tabular}
}

\maketitle 


\begin{abstract}
Machine learning is increasingly targeting areas where input data cannot be accurately described by a single vector, but can be modeled instead using the more flexible concept of \emph{random} vectors, namely probability measures or more simply point clouds of varying cardinality. 
Using deep architectures on measures poses, however, many challenging issues. Indeed, deep architectures are originally designed to handle fixed-length vectors, or, using recursive mechanisms, ordered sequences thereof. 
In sharp contrast, measures describe a varying number of weighted observations with no particular order. 
	We propose in this work a deep framework designed to handle crucial aspects of measures, namely permutation invariances, variations in weights and cardinality.
	Architectures derived from this pipeline can \begin{enumerate*}[label=(\roman*)]
	\item map measures to measures --- using the concept of push-forward operators;
	\item bridge the gap between measures and Euclidean spaces --- through integration steps.
\end{enumerate*} 
	This allows to design discriminative networks (to classify or reduce the dimensionality of input measures), generative architectures (to synthesize measures) and recurrent pipelines (to predict measure dynamics).
	We provide a theoretical analysis of these building blocks, review our architectures' approximation abilities and robustness w.r.t. perturbation, and try them on various discriminative and generative tasks. 

\end{abstract}


\section{Introduction}

Deep networks can now handle increasingly complex structured data types, starting historically from images~\citep{imagenet} and speech~\citep{speech} to deal now with shapes~\citep{shapenets}, sounds~\citep{audio}, texts~\citep{document} or graphs~\citep{graph}. In each of these applications, deep networks rely on the composition of several elementary functions, whose tensorized operations stream well on GPUs, and whose computational graphs can be easily automatically differentiated through back-propagation. Initially designed for vectorial features, their extension to \emph{sequences} of vectors using recurrent mechanisms~\citep{hochreiter1997long} had an enormous impact.  

Our goal is to devise neural architectures that can handle \emph{probability distributions} under any of their usual form: as discrete measures supported on (possibly weighted) point clouds, or densities one can sample from. 
	Such probability distributions are challenging to handle using recurrent networks because no order between observations can be used to treat them recursively (although some adjustments can be made, as discussed in~\citealt{44871}) and because, in the discrete case, their size may vary across observations. There is, however, a strong incentive to define neural architectures that can handle distributions as inputs or outputs.
	%
	This is particularly evident in computer vision, where the naive representation of complex 3D objects as vectors in spatial grids is often too costly memorywise, leads to a loss in detail, destroys topology and is blind to relevant invariances such as shape deformations. These issues were successfully tackled in a string of papers well adapted to such 3D settings~\citep{qi2016pointnet,qi2017pointnet++,fan2016point}. In other cases, ranging from physics \citep{godin2007measuring}, biology \citep{grover2011measuring}, ecology \citep{tereshko2000reaction} to census data \citep{guckenheimer1977dynamics}, populations cannot be followed at an individual level due to experimental costs or privacy concerns. In such settings where only macroscopic states are available, \textit{densities} appear as the right object to perform inference tasks.

\subsection{Previous works}
\paragraph{Specificities of Probability Distributions.} 


Data described in point clouds or sampled i.i.d. from a density are given \textit{unordered}. Therefore architectures dealing with them are expected to be \textit{permutation invariant}; they are also often expected to be equivariant to geometric transformations of input points (translations, rotations) and to capture \textit{local structures} of points. Permutation invariance or equivariance \citep{ravanbakhsh2016deep, ravanbakhsh2017deepsets}, or with respect to general groups of transformations \citep{gens2014symmetry, cohen2016group, ravanbakhsh2017equivariance} have been characterized, but without tackling the issue of locality. Pairwise interactions \citep{chen2014unsupervised, cheng2016deep, guttenberg2016permutation} are appealing and helpful in building permutation equivariant layers handling local information. Other strategies consist in augmenting the training data by all permutations or finding its \textit{best ordering} \citep{vinyals2015order}. \citep{qi2016pointnet, qi2017pointnet++} are closer to our work in the sense that they combine the search for local features to permutation invariance, achieved by max pooling. 

\paragraph{(Point) Sets \emph{vs.} Probability (Distributions).} An important distinction should be made between point \emph{sets}, and point \emph{clouds} which stand usually for discrete probability measures with uniform masses. The natural topology of (point) sets is the Hausdorff distance. That distance is very different from the natural topology for probability distributions, that of the convergence in law, \emph{a.k.a} the weak$^*$ topology of measures. The latter is metrized (among other metrics) by the Wasserstein (optimal transport) distance, which plays a key role in our work. This distinction between sets and probability is crucial, because the architectures we propose here are designed to capture stably and efficiently regularity of maps to be learned with respect to the convergence in law. Note that this is a crucial distinction between our work and that proposed in PointNet \citep{qi2016pointnet} and PointNet++ \citep{qi2017pointnet++}, which are designed to be smooth and efficients architectures for the Hausdorff topology of point sets. Indeed, they are \emph{not} continuous for the topology of measures (because of the max-pooling step) and cannot approximate efficiently maps which are smooth (e.g. Lipschitz) for the Wasserstein distance.

\paragraph{Centrality of optimal transport.}
The Wasserstein distance plays a central role in our architectures that are able to handle measures. Optimal transport has recently gained popularity in machine learning due to fast approximations, which are typically obtained using strongly-convex regularizers such as the entropy \citep{cuturi2013sinkhorn}. The benefits of this regularization paved the way to the use of OT in various settings \citep{courty2017optimal, rolet2016fast, huang2016supervised}. Although Wasserstein metrics have long been considered for inference purposes \citep{bassetti2006minimum}, their introduction in deep learning architectures is fairly recent, whether it be for generative tasks \citep{bernton:hal-01517550,arjovsky2017wasserstein,genevay2018learning} or regression purposes \citep{frogner2015learning, hashimoto2016learning}. The purpose of our work is to provide an extension of these works, to ensure that deep architectures can be used at a granulary level on measures directly. In particular, our work shares some of the goals laid out in \citep{hashimoto2016learning}, which considers recurrent architectures for measures (a special case of our framework). The most salient distinction with respect to our work is that our building blocks take into account multiple interactions between samples from the distributions, while their architecture has no interaction but takes into account diffusion through the injection of random noise.



\subsection{Contributions}

In this paper, we design deep architectures that can 
\begin{enumerate*}[label=(\roman*)]
	\item map measures to measures
	\item bridge the gap between measures and Euclidean spaces.
\end{enumerate*} 
They can thus accept as input for instance discrete distributions supported on (weighted) point clouds with an arbitrary number of points, can generate point clouds with an arbitrary number of points (arbitrary refined resolution) and are naturally invariant to permutations in the ordering of the support of the measure.
The mathematical idealization of these architectures are infinite dimensional by nature, and they can be computed numerically either by sampling (Lagrangian mode) or by density discretization (Eulerian mode). The Eulerian mode resembles classical convolutional deep network, while the Lagrangian mode, which we focus on, defines a new class of deep neural models.
Our first contribution is to detail this new framework for supervised and unsupervised learning problems over probability measures, making a clear connexion with the idea of iterative transformation of random vectors. These architectures are based on two simple building blocks: interaction functionals and self-tensorization. This machine learning pipeline works hand-in-hand with the use of optimal transport, both as a mathematical performance criterion (to evaluate smoothness and approximation power of these models) and  as a loss functional for both supervised and unsupervised learning.
Our second contribution is theoretical: we prove both quantitative Lipschitz robustness of these architectures for the topology of the convergence in law and universal approximation power. 
Our last contribution is a showcase of several instantiations of such deep stochastic networks for classification (mapping measures to vectorial features), generation (mapping back and forth measures to code vectors) and prediction (mapping measures to measures, which can be integrated in a recurrent network). 

%

\subsection{Notations} 

We denote $X \in \Rr(\RR^q)$ a random vector on $\RR^q$ and $\al_X \in \Mm_+^1(\RR^q)$ its law, which is a positive Radon measure with unit mass. It satisfies for any continuous map $f \in \Cc(\RR^q)$, $\EE(f(X)) = \int_{\RR^q} f(x) \d\al_X(x)$. Its expectation is denoted $\EE(X) = \int_{\RR^q} x \d\al_X(x) \in \RR^q$. 
We denote $\Cc(\RR^q;\RR^r)$ the space of continuous functions from $\RR^q$ to $\RR^r$ and $\Cc(\RR^q) \eqdef \Cc(\RR^q;\RR)$. 
In this paper, we focus on the \textit{law} of a random vector, so that two vectors $X$ and $X'$ having the same law (denoted $X \sim X'$) are considered to be indistinguishable.


\section{Stochastic Deep Architectures} \label{sec:description}

In this section, we define \emph{elementary blocks}, mapping random vectors to random vectors, which constitute a \emph{layer} of our proposed architectures, and depict how they can be used to build deeper networks.


\subsection{Elementary Blocks}

Our deep architectures  are defined by stacking a succession of simple elementary blocks that we now define.

\begin{defn}[Elementary Block] \label{defEB}
Given a function $f : \RR^q \times \RR^q \rightarrow \RR^r$, its associated elementary block 
$T_f : \Rr(\RR^q) \rightarrow \Rr(\RR^r)$ is defined as
\eql{\label{eq-elem-block-defn}
	\foralls X \in \Rr(\RR^q), \quad T_f(X) \eqdef \EE_{X'\sim X}( f(X,X') )
}
where $X'$ is a random vector independent from $X$ having the same distribution. 
\end{defn}

\begin{rem}[Discrete random vectors] 
A particular instance, which is the setting we use in our numerical experiments, is when $X$ is distributed uniformly on a set $(x_i)_{i=1}^n$ of $n$ points i.e. when $\al_X=\frac{1}{n}\sum_{i=1}^n \de_{x_i}$. In this case, $Y=T_f(X)$ is also distributed on $n$ points 
\eq{\label{eq-elem-block-defn-discr}
	\al_Y=\frac{1}{n}\sum_{i=1}^n \de_{y_i} \qwhereq 
	y_i=\frac{1}{n}\sum_{j=1}^n f(x_i,x_j).
} 
\end{rem}

This elementary operation~\eqref{eq-elem-block-defn-discr} displaces the distribution of $X$ according to pairwise interactions measured through the map $f$. 
As done usually in deep architectures, it is possible to localize the computation at some scale $\tau$ by imposing that $f(x,x')$ is zero for $\norm{x-x'} \geq \tau$, which is also useful to reduce the computation time. 




\begin{rem}[Fully-connected case] \label{rem-fully-connected}
%
As it is customary for neural networks, the map $f : \RR^{q} \times \RR^{q} \rightarrow \RR^{r}$ we consider for our numerical applications are affine maps composed by a pointwise non-linearity, i.e.
\eq{
	f(x,x') = ( \la( y_k  ) )_{k=1}^r 
	\qwhereq
	y = A \cdot [x;x'] + b \in \RR^r
}
where $\la : \RR \rightarrow \RR$ is a pointwise non-linearity (in our experiments, $\la(s)=\max(s,0)$ is the ReLu map). The parameter is then $\th=(A,b)$ where $A \in \RR^{ r \times 2q }$ is a matrix and $b \in \RR^{r}$ is a bias.  
\end{rem}

\begin{rem}[Deterministic layers] \label{random-det}
Classical ``deterministic'' deep architectures are recovered as special cases when $X$ is a constant vector, assuming some value $x \in \RR^q$ with probability 1, i.e. $\al_X=\de_x$.
A stochastic layer can output such a deterministic vector, which is important for instance for classification scores in supervised learning (see Section~\ref{sec:applications} for an example) or latent code vectors in auto-encoders (see Section~\ref{sec:applications} for an illustration).
In this case, the map $f(x,x')=g(x')$ does not depend on its first argument, so that $Y=T_f(X)$ is constant equal to $y=\EE_X(g(X))=\int_{\RR^q} g(x) \d\al_X(x)$.
Such a layer thus computes a summary statistic vector of $X$ according to~$g$.
\end{rem}

\begin{rem}[Push-Forward] \label{rmk-push-fwd}
	In sharp contrast to the previous remark, one can consider the case $f(x,x')=h(x)$ so that $f$ only depends on its first argument. One then has $T_f(X)=h(X)$, which corresponds to the notion of push-forward of measure, denoted $\al_{T_f(X)} = h_\sharp \al_X$. For instance, for a discrete law $\al_X=\frac{1}{n}\sum_i \de_{x_i}$ then $\al_{f_\sharp X}=\frac{1}{n}\sum_i \de_{h(x_i)}$. The support of the law of $X$ is thus deformed by $h$.
\end{rem}

\begin{rem}[Higher Order Interactions and Tensorization]
\label{rem-tensorization}
Elementary Blocks are generalized to handle higher-order interactions by considering $f : (\RR^q)^N \rightarrow \RR^r$, one then defines $T_f(X) \eqdef \EE_{X_2,\ldots,X_{N}}(f(X,X_2,\ldots,X_N))$ where $(X_2,\ldots,X_N)$ are independent and identically distributed copies of $X$. 
An equivalent and elegant way to introduce these interactions in a deep architecture is by adding a tensorization layer, which maps $X \mapsto X_2 \otimes \ldots \otimes X_N \in \Rr((\RR^q)^{N-1})$. 
Section~\ref{sec:theory} details the regularity and approximation power of these tensorization steps. 
\end{rem}


\subsection{Building Stochastic Deep Architectures}

These elementary blocks are stacked to construct deep architectures. A \emph{stochastic deep architecture} is thus a map
\eql{\label{eq-deep-archi}
	X \in \Rr(\RR^{q_0}) \mapsto Y= T_{f_T} \circ \cdots \circ T_{f_1} (X) \in \Rr(\RR^{q_{T}}), 
}
where $f_t : \RR^{q_{t-1}} \times \RR^{q_{t-1}} \to \RR^{q_{t}}$. Typical instances of these architectures includes: 
\begin{itemize}
	\item \textit{Predictive:} this is the general case where the architecture inputs a random vector and outputs another random vector. This is useful to model for instance time evolution using recurrent networks, and is used in Section~\ref{sec:applications} to tackle a dynamic prediction problem.
	\item \textit{Discriminative:} in which case $Y$ is constant equal to a vector $y \in \RR^{q_{T}}$ (i.e. $\al_{Y} = \de_{y}$) which can represent either a classification score or a latent code vector. Following Remark~\ref{random-det}, this is achieved by imposing that $f_T$ only depends on its second argument. 
	Section~\ref{sec:applications} shows applications of this setting to classification and variational auto-encoders (VAE). 
	\item \textit{Generative:} in which case the network should input a deterministic code vector $\tilde x_0 \in \RR^{\tilde q_0}$ and should output a random vector $Y$.
		This is achieved by adding extra randomization through a fixed random vector $\bar X_0 \in \Rr(\RR^{q_0-\tilde q_0})$ (for instance a Gaussian noise) and stacking $X_0=(\tilde x_0,\bar X_0) \in \Rr(\RR^{q_0})$. Section~\ref{sec:applications} shows an application of this setting to VAE generative models. Note that while we focus for simplicity on VAE models, it is possible to use our architectures for GANs~\cite{goodfellow2014generative} as well.  
\end{itemize}

\subsection{Recurrent Nets as Gradient Flows}
\label{sec-gradflows}

Following the work of~\cite{hashimoto2016learning}, in the special case $\RR^q=\RR^r$, one can also interpret iterative applications of such a $T_f$ (i.e. considering a recurrent deep network) as discrete optimal transport gradient flows~\citep{santambrogio2015optimal} (for the $W_2$ distance, see also Definition~\eqref{W_p}) in order to minimize a quadratic interaction energy $\Ee(\al) \eqdef \int_{\RR^q \times \RR^q} F(x,x') \d\al(x) \d\al(x')$ (we assume for ease of notation that $F$ is symmetric). 
Indeed, introducing a step size $\tau>0$, setting $f(x,x')=x + 2 \tau \nabla_x F(x,x')$, one sees that the measure $\al_{X_\ell}$ defined by the iterates $X_{\ell+1} = T_f(X_\ell)$ of a recurrent nets is approximating at time $t=\ell\tau$ the Wasserstein gradient flow $\al(t)$ of the energy $\Ee$. As detailed for instance in~\cite{santambrogio2015optimal}, such a gradient flow is the solution of the PDE $\pd{\al}{t} = \text{div}( \al \nabla(\Ee'(\al) )  )$ where $\Ee'(\al) = \int_{\RR^q} F(x,\cdot) \d\al(x)$ is the ``Euclidean'' derivative of $\Ee$.
The pioneer work of~\cite{hashimoto2016learning} only considers linear and entropy functionals of the form $\Ee(\al) = \int (F(x)+\log(\frac{\d\al}{\d x})) \d\al(x)$ which leads to evolutions $\al(t)$ being Fokker-Plank PDEs. Our work can thus be interpreted as extending this idea to the more general setting of interaction functionals (see Section~\ref{sec:theory} for the extension beyond pairwise interactions). 


\section{Theoretical Analysis} \label{sec:theory}

In order to get some insight on these deep architectures, we now highlight some theoretical results detailing the regularity and approximation power of these functionals. This theoretical analysis relies on the Wasserstein distance, which allows to make quantitative statements associated to the convergence in law.


\subsection{Convergence in Law Topology}

\paragraph{Wasserstein distance.} 

In order to measure regularity of the involved functionals, and also to define loss functions to fit these architectures (see Section~\ref{sec:applications}), we consider the $p$-Wasserstein distance (for $1 \leq p < +\infty$) between two probability distributions $(\al,\be) \in \Mm_+^1(\RR^q)$
\begin{equation} \label{W_p}
	\Wass_p^p(\al,\be) \eqdef \umin{ \pi_1=\al,\pi_2=\be } \int_{(\RR^q)^2} \norm{x-y}^p \d\pi(x,y)
\end{equation}
where $\pi_1,\pi_2 \in \Mm_+^1(\RR^q)$ are the two marginals of a coupling measure $\pi$, and the minimum is taken among coupling measures $\pi \in  \Mm_+^1(\RR^q \times \RR^q)$. 

A classical result (see~\cite{santambrogio2015optimal}) asserts that $\Wass_{1}$ is a norm, and can be conveniently computed using
\eq{
	\Wass_1(\al,\be) = \Wass_1(\al-\be) = \umax{\Lip(g) \leq 1} \int_\Xx g \d(\al-\be),
}
where $\Lip(g)$ is the Lipschitz constant of a map $g : \Xx \rightarrow \RR$ (with respect to the Euclidean norm unless otherwise stated).

With an abuse of notation, we write $\Wass_p(X,Y)$ to denote $\Wass_p(\al_X,\al_Y)$, but one should be careful that we are considering distances between laws of random vectors. An alternative formulation is $\Wass_p(X,Y) = \min_{X',Y'} \EE(\norm{X'-X'}^p)^{1/p}$ where $(X',Y')$ is a couple of vectors such that $X'$ (resp. $Y'$) has the same law as $X$ (resp. $Y$), but of course $X'$ and $Y'$ are not necessarily independent.
The Wasserstein distance metrizes the convergence in law (denoted $\rightharpoonup$) in the sense that $X_k \rightharpoonup X$ is equivalent to $\Wass_1(X_k,X) \rightarrow 0$. 

In the numerical experiments, we estimate $\Wass_p$ using Sinkhorn's algorithm~\citep{cuturi2013sinkhorn}, which provides a smooth approximation amenable to (possibly stochastic) gradient descent optimization schemes, whether it be for generative or predictive tasks (see Section~\ref{sec:applications}). 

\paragraph{Lipschitz property.} 

A map $T : \Rr(\RR^q) \rightarrow \Rr(\RR^r)$ is continuous for the convergence in law (aka the weak$^*$ of measures) if for any sequence $X_k \rightharpoonup X$, then $T(X_k) \rightharpoonup T(X)$. 
Such a map is furthermore said to be $C$-Lipschitz for the 1-Wasserstein distance if 
\eql{\label{eq-wass-lip}
	\foralls (X,Y) \in \Rr(\RR^q)^2, \, 
	\Wass_1( T(X),T(Y) )
	\leq 
	C \Wass_1(X,Y).
}
Lipschitz properties enable us to analyze robustness to input perturbations, since it ensures that if the input distributions of random vectors are close enough (in the Wasserstein sense), the corresponding output laws are close too.

\subsection{Regularity of Building blocks} \label{building-blocks}

\paragraph{Elementary blocks.} The following proposition, whose proof can be found in Appendix~\ref{sec-proof-propEB}, shows that elementary blocks are robust to input perturbations. 

\begin{prop}[Lipschitzianity of elementary blocks] \label{propEB}
If for all $x$,  $f(x,\cdot)$ and $f(\cdot,x)$ are $C(f)$-Lipschitz, then $T_f$ is $2rC(f)$-Lipschitz in the sense of~\eqref{eq-wass-lip}.
\end{prop}

As a composition of Lipschitz functions defines Lipschitz maps, the architectures of the form~\eqref{eq-deep-archi} are thus Lipschitz, with a Lipschitz constant upper-bounded by $2r \sum_t C(f_t)$, where we used the notations of Proposition \ref{propEB}.

\paragraph{Tensorization.}

As highlighted in Remark~\ref{rem-tensorization}, tensorization plays an important role to define higher-order interaction blocks. 

\begin{defn}[Tensor product]
Given two $(X,Y) \in \Rr(\Xx) \times \Rr(\Yy)$, a tensor product random vector is $X \otimes Y \eqdef (X',Y') \in \Rr( \Xx \times \Yy )$ where $X'$ and $Y'$ are independent and have the same law as $X$ and $Y$. This means that $\d\al_{X \otimes Y}(x,y) = \d\al_X(x)\d\al_Y(y)$ is the tensor product of the measures. 
\end{defn}


\begin{rem}[Tensor Product between Discrete Measures]
	If we consider random vectors supported on point clouds, with laws $\al_X = \frac{1}{n} \sum_{i=1}^n \de_{x_i}$ and $\al_Y = \frac{1}{m} \sum_{j=1}^m \de_{y_j}$, then $X \otimes Y$ is a discrete random vector supported on $nm$ points, since $\al_{X \otimes Y} = \frac{1}{nm}\sum_{i,j} \de_{(x_i,y_j)}$.
\end{rem}

The following proposition shows that tensorization blocks maintain the stability property of a deep architecture.

\begin{prop}[Lipschitzness of tensorization]\label{prop-tensorization}
	One has, for $(X,X',Y,Y') \in \Rr(\Xx)^2 \times \Rr(\Yy)^2$, 
	\eq{
		\Wass_1( X \otimes Y,  X' \otimes Y' ) \leq \Wass_1(X,X') + \Wass_1(Y,Y').
	}
\end{prop}
\begin{proof}
	One has
	\eq{
	\begin{split}
		 &\Wass_{1}( \al \otimes \be,  \al' \otimes \be' )  \\
		&= \umax{\Lip(g)\leq 1}
			\int_{1} g(x,y) [ \d\al(x)\d\be(y)-\d\al'(x)\d\be'(y) ] \\
		& = 
		\umax{\Lip(g)\leq 1}
		\int_\Xx \int_\Yy g(x,y)[\d\be(y)-\d\be'(y)]\d\al(x) \\
		& +   
		\int_\Yy \int_\Xx g(x,y)[\d\al(x)-\d\al'(x)]\d\be(y),
	\end{split}
	}
	hence the result.
\end{proof}

\subsection{Approximation Theorems}

\paragraph{Universality of elementary block.}

The following theorem shows that any continuous map between random vectors can be approximated to arbitrary precision using three  elementary blocks. Note that it includes through $\La$ a fixed random input which operates as an ``excitation block'' similar to the generative VAE models studied in Section~\ref{sec:generation}.

\begin{thm}\label{thm-meas-valued}
	Let $\Ff : \Rr(\Xx) \rightarrow \Rr(\Yy)$ be a continuous map for the convergence in law, where $\Xx \subset \RR^q$ and $\Yy \subset \RR^r$ are compact. Then $\forall \epsilon > 0$ there exists three continuous maps $f, g, h$ such that 
	\eql{\label{eq-approx-thm}
		\forall X \in \Rr(\Xx), \quad \Wass_1 ( \Ff(X), T_h \circ \Lambda \circ T_{g} \circ T_{f}(X) ) \leq \epsilon.
	}
	where $\Lambda : X \mapsto (X,U)$ concatenates a uniformly distributed random vector $U$. 
\end{thm}

The architecture that we use to prove this theorem is displayed on Figure~\ref{fig:example-archi}, bottom (left). 
Since $f$, $g$ and $h$ are smooth maps, according to the universality theorem of neural networks~\citep{cybenko1989approximation,leshno1993multilayer} (assuming some restriction on the non-linearity $\la$, namely its being a nonconstant, bounded and continuous function), it is possible to replace each of them (at the expense of increasing $\epsilon$) by a sequence of fully connected layers (as detailed in Remark~\ref{rem-fully-connected}). This is detailed in Section~\ref{sec-proof-approx-thm} of the appendix. 

Since deterministic vectors are a special case of random vectors (see Remark~\ref{random-det}), this results encompasses as a special case universality for vector-valued maps  $\Ff : \Rr(\Omega) \rightarrow \RR^r$ (used for instance in classification in Section~\ref{sec:classif}) and in this case only 2 elementary blocks are needed. Of course the classical universality of multi-layer perceptron~\citep{cybenko1989approximation,leshno1993multilayer} for vectors-to-vectors maps $\Ff : \RR^q \rightarrow \RR^r$ is also a special case (using a single elementary block). 

\begin{proof}(of Theorem~\ref{thm-meas-valued})
In the following, we denote the probability simplex as $\Sigma_n = \enscond{ a \in \RR_+^n }{\sum_i a_i=1}$.
Without loss of generality, we assume $\Xx \subset [0,1]^q$ and $\Yy \subset [0,1]^r$. We consider two uniform grids of $n$ and $m$ points $(x_i)_{i=1}^{n}$ of $[0,1]^q$ and $(y_j)_{j=1}^{m}$ of $[0,1]^r$.  
On these grids, we consider the usual piecewise affine P1 finite element bases $(\phi_i)_{i=1}^{n}$ and $(\psi_j)_{j=1}^{m}$, which are continuous hat functions supported on cells $(R_i)_i$ and $(S_j)_j$ which are cubes of width $2/n^{1/q}$ and $2/m^{1/r}$. 
We define discretization operators as $D_\Xx : \al \in \Mm_1^+(\Xx) \mapsto ( \int_{R_i} \phi_i \d \al )_{i=1}^{n} \in \Si_{n}$
and $D_\Yy : \be \in \Mm_1^+(\Yy) \mapsto ( \int_{S_j} \psi_j \d \be )_{j=1}^{m} \in \Si_{n}$. We also define $D_\Xx^* : a \in \Si_n \mapsto \sum_i a_i \de_{x_i} \in \Mm_1^+(\Xx)$ and $D_\Yy^* : b \in \Si_m \mapsto \sum_j b_j \de_{y_i} \in \Mm_1^+(\Yy)$.

The map $\Ff$ induces a discrete map $G : \Si_{n} \rightarrow \Si_{m}$ defined by $G \eqdef D_\Yy \circ \Ff \circ D_\Xx^*$. Remark that $D_\Xx^*$ is continuous from $\Si_n$ (with the usual topology on $\RR^n$) to $\Mm_+^1(\Xx)$ (with the convergence in law topology), $\Ff$ is continuous (for the convergence in law), $D_\Yy$ is continuous from $\Mm_+^1(\Yy)$ (with the convergence in law topology) to $\Si_m$ (with the usual topology on $\RR^m$). This shows that $G$ is continuous. 

For any $b \in \Sigma_{m}$, Lemma~\ref{lemma-noise-reshaping} proved in the appendices defines a continuous map $H$ so that, defining $U$ to be a random vector uniformly distributed on $[0,1]^r$ (with law $\Uu$), $H(b,U)$ has law $ (1-\epsilon) D_\Yy^*(b) + \epsilon \Uu$. 

We now have all the ingredients, and define the three continuous maps for the elementary blocks as
\eq{
	f(x,x') = (\phi_i(x'))_{i=1}^n \in \RR^n, \quad
	g(a,a') = G(a') \in \RR^m, 
}
\eq{
	\qandq 
	h((b,u),(b',u')) = H(b,u) \in \Yy.
}
The corresponding architecture is displayed on Figure~\ref{fig:example-archi}, bottom. 
Using these maps, one needs to control the error between $\Ff$ and $\hat \Ff \eqdef T_h \circ \Lambda \circ T_{g} \circ T_{f} = H_\sharp \circ \Lambda \circ D_\Yy \circ \Ff \circ D_\Xx^* \circ \Dd_{\Xx}$ where we denoted $H_\sharp(b) \eqdef H(b,\cdot)_\sharp \Uu$ the law of $H(b,U)$ (i.e. the pushforward of the uniform distribution $\Uu$ of $U$ by $H(b,\cdot)$). 

(i) We define $\hat\al \eqdef D_\Xx^*\Dd_\Xx(\al)$. The diameters of the cells $R_i$ is $\De_j = \sqrt{q}/n^{1/q}$, so that Lemma~\ref{lem-proof-approx-thm} in the appendices shows that $\Wass_1( \al,\hat\al) \leq \sqrt{q}/n^{1/q}$. Since $\Ff$ is continuous for the convergence in law, choosing $n$ large enough ensures that $\Wass_1(\Ff(\al),\Ff(\hat\al)) \leq \epsilon$.

(ii) We define $\hat\be \eqdef D_\Yy^* D_\Yy\Ff(\hat\al)$. Similarly, using $m$ large enough ensures that $\Wass_1(\Ff(\hat \al),\hat\be) \leq \epsilon$.

(iii) Lastly, let us define $\tilde\be \eqdef H_\sharp \circ D_\Yy(\hat\be)=\hat\Ff(\al)$. By construction of the map $H$ in Lemma~\ref{lemma-noise-reshaping}, one has hat $\tilde\be = (1-\epsilon) \hat\be + \epsilon \Uu$ so that $\Wass_1(\tilde\be,\hat\be) = \epsilon \Wass_1(\hat \be, \Uu) \leq C \epsilon$ for the constant $C=2 \sqrt{r}$ since the measures are supported in a set of diameter $\sqrt{r}$.

Putting these three bounds (i), (ii) and (iii) together using the triangular inequality shows that $\Wass_1(\Ff(\al),\hat\Ff(\al)) \leq (2+C)\epsilon$. 
\end{proof}

\paragraph{Universality of tensorization.}

The following Theorem, whose proof can be found in Appendix~\ref{sec-proof-approx-tenso}, shows that one can approximate any continuous map using a high enough number of tensorization layers followed by an elementary block.

\begin{thm}\label{thm-density-tensorization}
	Let $\Ff : \Rr(\Omega) \rightarrow \RR$ a continuous map for the convergence in law, where $\Omega \subset \RR^q$ is compact. Then $\forall \epsilon > 0$, there exists $n>0$ and a continuous function $f$ such that
	\eql{\label{eq-approx-tensor}
		\foralls X \in \Rr(\Om), \quad
		|\Ff(X) - T_f \circ \th_n( X ) | \leq \epsilon
	}
	where $\th_n(X) = X \otimes \ldots \otimes X$ is the $n$-fold self tensorization. 
\end{thm}

The architecture used for this theorem is displayed on the bottom (right) of Figure~\ref{fig:example-archi}. The function $f$ appearing in~\eqref{eq-approx-tensor} plays a similar role as in~\eqref{eq-approx-thm}, but note that the two-layers factorizations provided by these two theorems are very different. It is an interesting avenue for future work to compare them theoretically and numerically.


\section{Applications} \label{sec:applications}

To exemplify the use of our stochastic deep architectures, we consider classification, generation and dynamic prediction tasks. The goal is to highlight the versatility of these architectures and their ability to handle as input and/or output both probability distributions and vectors. We also illustrate the gain in maintaining stochasticity among several layers. 
%

\begin{figure*}[h!]
	\centering
	\includegraphics[width=\linewidth]{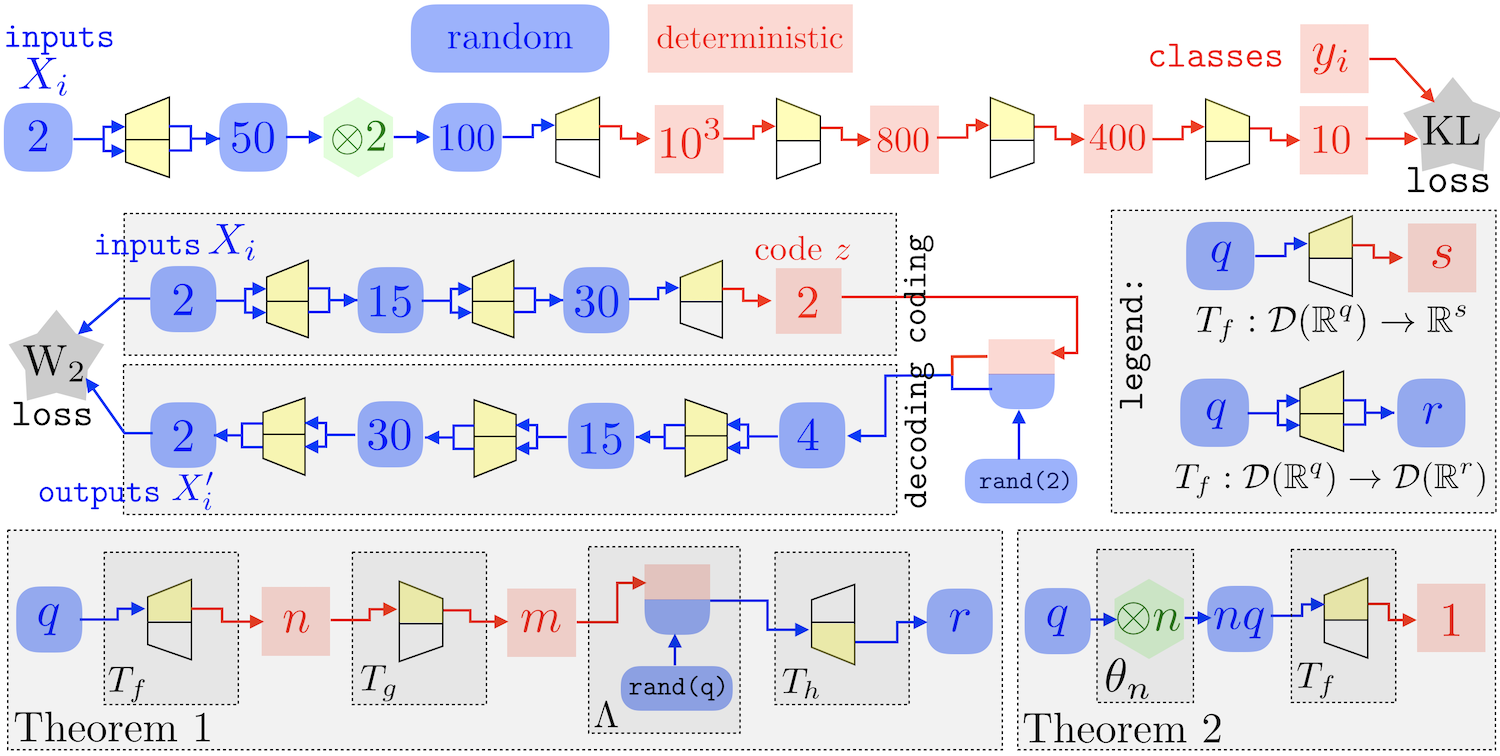}
	\caption{Top and center: two examples of deep stochastic architectures applied to the MNIST dataset: top for classification purpose (Section~\ref{sec:classif}), center for generative model purpose (Section~\ref{sec:generation}). 
		Bottom: architecture for the proof of Theorems~\ref{thm-meas-valued} and~\ref{thm-density-tensorization}. 
	\label{fig:example-archi}}
\end{figure*}

\subsection{Classification tasks}
\label{sec:classif}

\paragraph{MNIST Dataset.}

We perform classification on the 2-D MNIST dataset of handwritten digits. To convert a MNIST image into a 2D point cloud, we threshold pixel values (threshold $\rho = 0.5$) and use as a support of the input empirical measure the $n=256$ pixels of highest intensity, represented as points $(x_i)_{i=1}^n \subset \RR^2$ (if there are less than $n=256$ pixels of intensity over $\rho$, we repeat input coordinates), which are  remapped  along each axis by mean and variance normalization. Each image is therefore turned into a sum of $n=256$ Diracs $\frac{1}{n}\sum_i \delta_{x_i}$. 
Our stochastic network architecture is displayed on the top of Figure~\ref{fig:example-archi} and is composed of 5 elementary blocks $(T_{f_k})_{k=1}^5$ with an interleaved self-tensorisation layer $X \mapsto X \otimes X$. The first elementary block $T_{f_1}$ maps measures to measures, the second one $T_{f_2}$ maps a measure to a deterministic vector (i.e. does not depend on its first coordinate, see Remark~\ref{random-det}), and the last layers are classical vectorial fully-connected ones. We use a ReLu non-linearity $\la$ (see Remark~\ref{rem-fully-connected}). The weights are learnt with a weighted cross-entropy loss function over a training set of 55,000 examples and tested on a set of 10,000 examples. 
Initialization is performed through the Xavier method~\citep{glorot2010understanding} and learning with the Adam optimizer~\citep{kingma2014adam}. Table~\ref{tab:mnistClassif} displays our results, compared with the PointNet~\cite{qi2016pointnet} baseline. We observe that maintaining stochasticity among several layers is beneficial (as opposed to replacing one Elementary Block with a fully connected layer allocating the same amount of memory).

\begin{table}[H]
\centering
 \begin{tabular}{c c c}
         & input type & error ($\%$)   \\  \hline 
		 PointNet & point set & 0.78 \\ \hline
		 Ours & measure (1 stochastic layer) & 1.07 \\ 
		 Ours & measure (2 stochastic layers) & \textbf{0.76} \\ \hline
    \end{tabular}
\vspace*{-10pt}
\caption{MNIST classification results} \label{tab:mnistClassif}
\vspace*{-20pt}
  \end{table}
	
\paragraph{ModelNet40 Dataset.}

We evaluate our model on the ModelNet40~\citep{wu20153d} shape classification benchmark. The dataset contains 3-D CAD models from 40 man-made categories, split into 9,843 examples for training and 2,468 for testing. We consider $n=1,024$ samples on each surface, obtained by a farthest point sampling procedure.  
%
Our classification network is similar to the one displayed on top of Figure~\ref{fig:example-archi}, excepted that the layer dimensions are 
$[{\color{blue} 3},
{\color{blue}10},
{\color{red}500},
{\color{red}800},
{\color{red}400},
{\color{red}40}]$.
Our results are displayed in figure \ref{tab:modelnetClassif}. As previously observed in 2D, peformance is improved by maintaining stochasticity among several layers, for the same amount of allocated memory.

\begin{table}[H]
\centering
 \begin{tabular}{c c c}
         & input type & accuracy ($\%$)   \\  \hline 
     3DShapeNets & volume  & 77 \\ \hline
		 Pointnet & point set & 89.2 \\ \hline
		 Ours & measure & 82.0 \\
		  & (1 stochastic layer) & \\
		 Ours & measure & \textbf{83.5} \\
		  & (2 stochastic layers) & \\ \hline
    \end{tabular}
\vspace*{-10pt}
\caption{ModelNet40 classification results} \label{tab:modelnetClassif}
\vspace*{-20pt}
  \end{table}

\subsection{Generative networks}
\label{sec:generation}

We further evaluate our framework for generative tasks, on a VAE-type model~\citep{kingma2013auto} (note that it would be possible to use our architectures for GANs~\cite{goodfellow2014generative} as well). The task consists in generating outputs resembling the data distribution by decoding a random variable $z$ sampled in a latent space $\Zz$. The model, an encoder-decoder architecture, is learnt by comparing input and output measures using the $\Wass_2$ Wasserstein distance loss, approximated using Sinkhorn's algorithm~\citep{cuturi2013sinkhorn,genevay2018learning}. 
Following~\citep{kingma2013auto}, a Gaussian prior is imposed on the latent variable $z$. The encoder and the decoder are two mirrored architectures composed of two elementary blocks and three fully-connected layers each. The corresponding stochastic network architecture is displayed on the bottom of~\ref{fig:example-archi}.
Figure~\ref{fig:VAEmist} displays an application on the MNIST database where the latent variable $z \in \RR^2$ parameterizes a 2-D of manifold of generated digits. We use as input and output discrete probability measures of $n=100$ Diracs, displayed as point clouds on the right of Figure~\ref{fig:VAEmist}.

\begin{figure}
	\centering
	\begin{tabular}{@{}c@{\hspace{1mm}}c@{}}
	 \includegraphics[width=.47\linewidth]{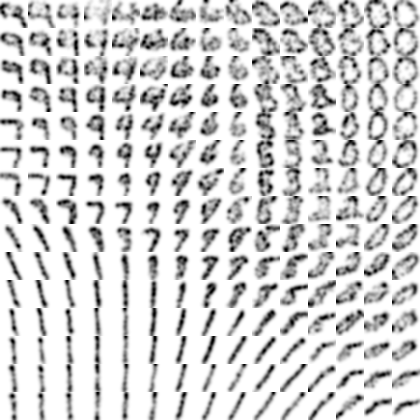}  & 
	 \includegraphics[width=.49\linewidth]{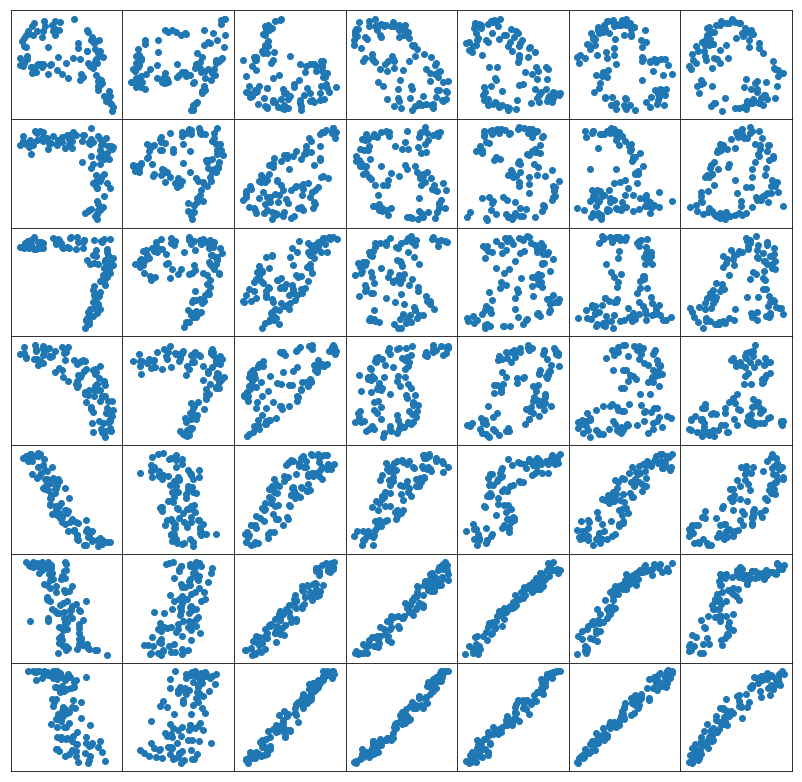}
	\end{tabular}
	\caption{
		Left: Manifold of digits generated by the VAE network displayed on the bottom of~\ref{fig:example-archi}.
		Right: Corresponding point cloud (displaying only a subset of the left images).  
\label{fig:VAEmist}}
\end{figure}

\subsection{Dynamic Prediction}



The Cucker-Smale flocking model~\citep{cucker2007mathematics} is non-linear dynamical system modelling the emergence of coherent behaviors, as for instance in the evolution of a flock of birds, by solving for positions and speed $x_i(t) \eqdef (p_i(t) \in \RR^d, v_i(t) \in \RR^d)$ for $i=1,\ldots,n$
\eql{\label{eq-cs-model}
	\dot p(t)=v(t), 
	\qandq
	\dot v(t) = \Ll(p(t)) v(t)
} 
where $\Ll(p) \in \RR^{n \times n}$ is the Laplacian matrix associated to a group of points $p \in (\RR^d)^n$
\eq{
	\Ll(p)_{i,j} \eqdef \frac{1}{1+\norm{p_i-p_j}^m}, \,
	\Ll(p)_{i,i} = -\sum_{j \neq i} \Ll(p)_{i,j}.
}
In the numerics, we set $m=0.6$. 
This setting can be adapted to \emph{weighted} particles $(x_i(t),\mu_i)_{i=1\cdots n}$, where each weight $\mu_i$ stands for a set of physical attributes impacting dynamics -- for instance, mass -- which is what we consider here.
This model equivalently describes the evolution of the measure $\al(t)=\sum_{i=1}^n \mu_i \de_{x_i(t)}$ in phase space $(\RR^d)^2$, and following Remark~\ref{sec-gradflows} on the ability of our architectures to model dynamical system involving interactions,~\eqref{eq-cs-model} can be discretized in time which leads to a recurrent network making use of a single elementary block $T_f$ between each time step. Indeed, our block allows to maintain stochasticity among all layers -- which is the natural way of proceeding to follow densities of particles over time.

It is however not the purpose of this article to study such a recurrent network and we aim at showcasing here whether deep (non-recurrent) architectures of the form~\eqref{eq-deep-archi} can accurately capture the Cucker-Smale model. More precisely, since in the evolution~\eqref{eq-cs-model} the mean of $v(t)$ stays constant, we can assume $\sum_i v_i(t)=0$, in which case it can be shown~\citep{cucker2007mathematics} that particles ultimately reach stable positions $(p(t),v(t)) \mapsto (p(\infty),0)$. We denote $\Ff(\al(0)) \eqdef \sum_{i=1}^n \mu_i \de_{p_i(\infty)}$ the map from some initial configuration in the phase space (which is described by a probability distribution $\al(0)$) to the limit probability distribution (described by a discrete measure supported on the positions $p_i(\infty)$). The goal is to approximate this map using our deep stochastic architectures. To showcase the flexibility of our approach, we consider a \emph{non-uniform} initial measure $\al(0)$ and approximate its limit behavior $\Ff(\al(0))$ by a uniform one ($\mu_i=\frac{1}{n}$).

In our experiments, the measure $\al(t)$ models the dynamics of several (2 to 4) flocks of birds moving towards each other, exhibiting a limit behavior of a single stable flock. As shown in Figures~\ref{fig:flockingDensity} and~\ref{fig:flockingPoints}, positions of the initial flocks are normally distributed, centered respectively at edges of a rectangle $(-4;2), (-4;-2), (4;2), (4;-2)$ with variance 1. Their velocities (displayed as arrows with lengths proportional to magnitudes in Figures~\ref{fig:flockingDensity} and~\ref{fig:flockingPoints}) are uniformly chosen within the quarter disk $[0;-0.1]\times [0.1;0]$. Their initial weights $\mu_i$ are normally distributed with mean 0.5 and sd 0.1, clipped by a ReLu and normalized. Figures~\ref{fig:flockingDensity} (representing densities) and~\ref{fig:flockingPoints} (depicting corresponding points' positions) show that for a set of $n=720$ particles, quite different limit behaviors are successfully retrieved by a simple network composed of {\color{red}five} elementary blocks with layers of dimensions $[{\color{blue} 2},
{\color{blue}10},
{\color{blue}20},
{\color{blue}40},
{\color{blue}60}]$, learnt with a Wasserstein~\citep{genevay2018learning} fitting criterion (computed with Sinkhorn's algorithm~\citep{cuturi2013sinkhorn}).

\newcommand{\sidecapY}[1]{ \begin{sideways}\parbox{.28\linewidth}{\centering #1}\end{sideways} }
\newcommand{\FigGF}[1]{\includegraphics[width=.33\linewidth]{img/flocking-6/limit#1}}
\newcommand{\FigRowGF}[1]{\FigGF{#1-init} & \FigGF{#1-true} & \FigGF{#1-predicted}}
\newcommand{\FigPointsGF}[1]{\includegraphics[width=.33\linewidth]{img/flocking-6/points#1}}
\newcommand{\FigRowPointsGF}[1]{\FigPointsGF{#1-init} & \FigPointsGF{#1-true} & \FigPointsGF{#1-predicted}}

\begin{figure}[h!]
	\centering
	\begin{tabular}{@{}c@{}c@{}c@{}c@{}}
	 \sidecapY{ $(a)$ }  & \FigRowGF{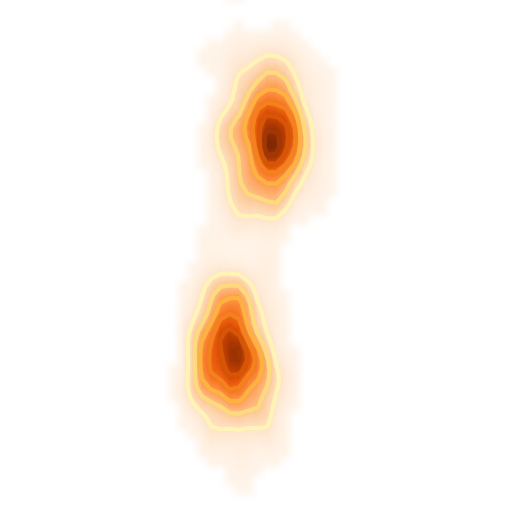} \\
	 \sidecapY{ $(b)$ }  & \FigRowGF{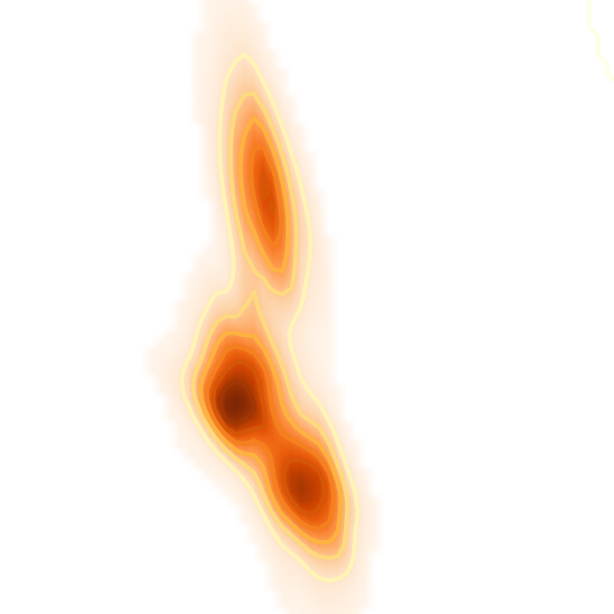} \\
	 \sidecapY{ $(c)$ }  & \FigRowGF{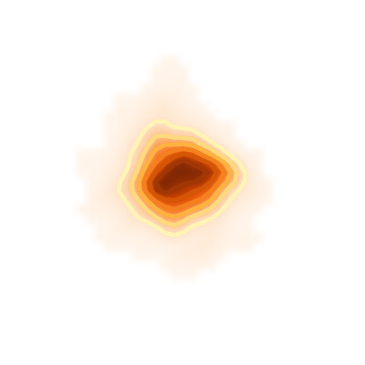} \\
	 & Initial $\al(0)$ &  $\Ff(\al(0))$ & Predicted 
	\end{tabular}
	\caption{
		Prediction of the asymptotic density of the flocking model, for various initial speed values $v(0)$ and $n=720$ particles. Eg. for top left cloud: (a) $v(0)=(0.050;-0.085)$; (b) $v(0)=(0.030;-0.094)$; (c) $v(0)=(0.056;-0.081)$.
		}
\label{fig:flockingDensity}
\end{figure}

\begin{figure}[h!]
	\centering
	\begin{tabular}{@{}c@{}c@{}c@{}c@{}}
	 \sidecapY{ $(a)$ }  & \FigRowPointsGF{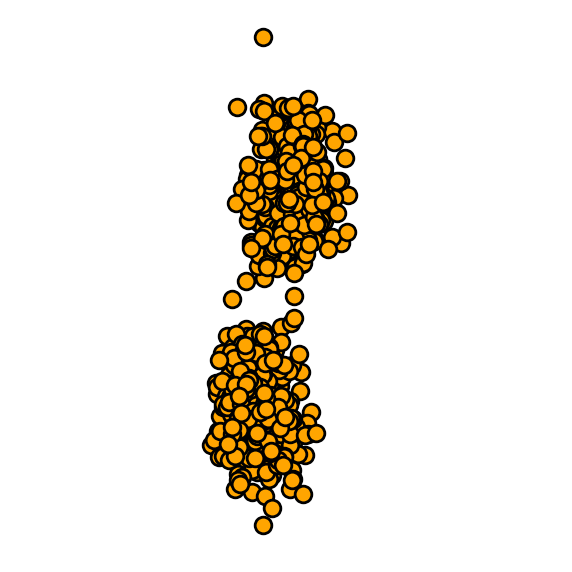} \\
	 \sidecapY{ $(b)$ }  & \FigRowPointsGF{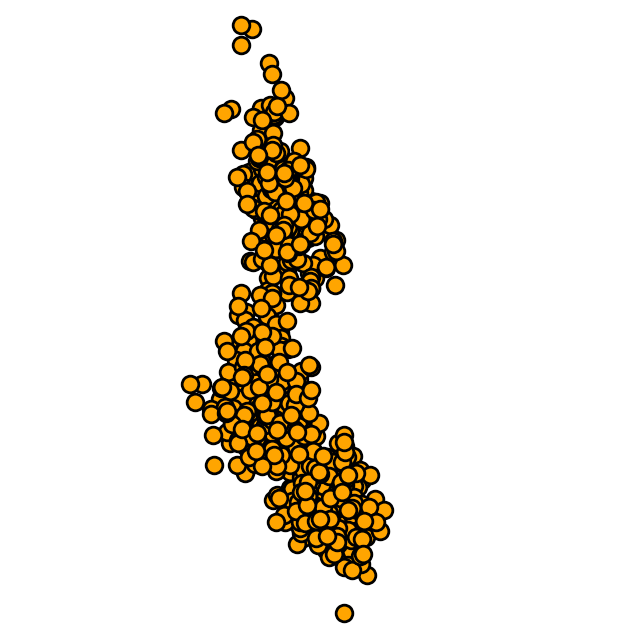} \\
	 \sidecapY{ $(c)$ }  & \FigRowPointsGF{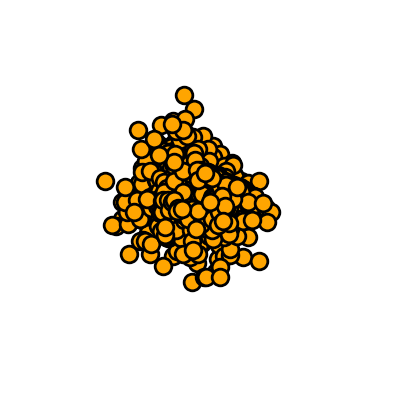} \\
	 & Initial $\al(0)$ &  $\Ff(\al(0))$ & Predicted 
	\end{tabular}
	\caption{
		Prediction of particles' positions corresponding to Figure~\ref{fig:flockingDensity}. Dots' diameters are proportial to weights $\mu_i$ (predicted ones all have the same size since $\mu_i=\frac{1}{n}$).
\label{fig:flockingPoints}}
\end{figure}


\section*{Conclusion}

In this paper, we have proposed a new formalism for stochastic deep architectures, which can cope in a seamless way with both probability distributions and deterministic feature vectors. The salient features of these architectures are their robustness and approximation power with respect to the convergence in law, which is crucial to tackle high dimensional classification, regression and generation problems over the space of probability distributions.

\section*{Acknowledgements}

The work of Gwendoline de Bie was supported by a grant from R\'egion Ile-de-France.
The work of Gabriel Peyr\'e has been supported by the European Research Council (ERC project NORIA).

\bibliographystyle{apalike}
\bibliography{biblio}

\newpage

\appendix


\newpage
\section*{Supplementary}

\section{Useful Results}

We detail below operators which are useful to decompose a building block $T_f$ into push-forward and integration.
In the following, to ease the reading, we often denote $\Wass_\Xx$ the $\Wass_1$ distance on a space $\Xx$ (e.g. $\Xx=\RR^q$).



\paragraph{Push-forward.} 

The push-forward operator allows for modifications of the support while maintaining the geometry of the input measure.
\begin{defn}[Push-forward]
For a function $f : \Xx \rightarrow \Yy$, we define the push-forward  $f_\sharp \al \in \Mm(\Yy)$ of $\al \in \Mm(\Xx)$ by $T$ as defined by
\eql{\label{eq-push-foward}
	\foralls g \in \Cc(\Yy), \quad 
	\int_\Yy g \d(f_\sharp \al) \eqdef \int_\Xx g(f(x)) \d\al(x).	
}
\end{defn}
\noindent Note that $f_\sharp : \Mm(\Xx) \rightarrow \Mm(\Yy)$ is a linear operator.

\begin{prop}[Lipschitzness of push-forward]\label{prop-lip-pushforward}
	One has
	\begin{align}
		\Wass_\Yy(f_\sharp \al,f_\sharp \be) &\leq \Lip(f) \Wass_\Xx(\al,\be), \label{lip_push_1} \\
		\Wass_\Yy(f_\sharp \al,g_\sharp \al) &\leq \norm{f-g}_{L^1(\al)} \label{lip_push_2}, 
	\end{align}
	where $\Lip(f)$ designates the Lipschitz constant of $f$.
\end{prop}

\begin{proof}
$\forall h:\Yy\to\RR$ s.t. $\Lip(h) \leq 1, \frac{h \circ f}{\text{Lip}(f)}$ is 1-Lipschitz, therefore
$$ \int_\Xx \frac{h \circ f}{\Lip(f)} \mathrm{d}(\al-\be) \leq \Wass_\Xx(\al,\be)
$$
hence inequality (\ref{lip_push_1}). Similarly, $\forall h$ s.t. $\Lip(h) \leq 1,$
$$ \int_\Xx (h \circ f - h \circ g) \mathrm{d}\al \leq \int_\Xx \norm{f(x)-g(x)}_{2} \mathrm{d}\al(x) $$
hence inequality (\ref{lip_push_2}).
\end{proof}

\paragraph{Integration.}

We now define a (partial) integration operator.

\begin{defn}[Integration]
For $f \in \Cc(\Zz \times \Xx; \Yy=\RR^r)$, and $\al \in \Mm(\Xx)$ we denote
\eq{
	f[\cdot,\al] \eqdef \int_{\Xx} f(\cdot, x) \d \al(x) : \Zz \rightarrow \Yy.
}
\end{defn}

\begin{prop}[Lipschitzness of integration]\label{prop-lip-integration}
	With some fixed $\zeta \in \Mm_+^1(\Zz)$, one has
	\begin{align*}
		\norm{f[\cdot,\al] - f[\cdot,\be]}_{L^1(\zeta)} \leq r \Lip_2(f) \Wass_\Xx(\al,\be).
	\end{align*}
	where we denoted by $\Lip_2(f)$ a bound on the Lipschitz contant of the function $f(z,\cdot)$ for all $z$.
\end{prop}

\begin{proof}
\eq{
\begin{split}
\lvert \lvert f[\cdot,\al] & - f[\cdot,\be] \lvert \lvert_{L^1(\zeta)} \\
& = \int_\Zz \| f[\cdot,\al](z) - f[\cdot,\be](z) \|_2 \d\zeta(z) \\
& = \int_\Zz \left\| \int_\Xx f(z,x) \d(\al - \be)(x) \right\|_2 \d\zeta(z) \\
& \leq \int_\Zz \left\| \int_\Xx f(z,x) \d(\al - \be)(x) \right\|_1 \d\zeta(z) \\
& = \int_\Zz \sum_{i=1}^r \left| \int_\mathcal{X} f_i(z,x) \d(\al-\be)(x) \right| \d\zeta(z) \\
& \leq \sum_{i=1}^r \Lip_2(f_i) \Wass_\Xx(\al,\be) \\
& \leq r \Lip_2(f) \Wass_\Xx(\al,\be)
\end{split}
}
where we denoted by $\Lip_2(f_i)$ a bound on the Lipschitz contant of the function $f_i(z,\cdot)$ ($i$-th component of $f$) for all $z$, since again, $\frac{f_i}{\Lip_2(f_i)}$ is 1-Lipschitz.
\end{proof}

\paragraph{Approximation by discrete measures}

The following lemma shows how to control the approximation error between an arbitrary random variable and a discrete variable obtained by computing moments agains localized functions on a grid. 

\begin{lem}\label{lem-proof-approx-thm}
	Let $\left( S_j \right)_{j=1}^N$ be a partition of a domain including $\Omega$ ($S_j \subset \RR^d$) and let $x_j \in S_j$.
	Let $( \phi_j )_{j=1}^N$ a set of bounded functions $\phi_j : \Om \rightarrow \RR$ supported on $S_j$, such that $\sum_j \phi_j=1$ on $\Om$. 
	For $\al \in \Mm_+^1(\Omega)$, we denote $\hat{\al}_N \eqdef \sum_{j=1}^N \al_j \delta_{x_j}$ with $\al_j \eqdef \int_{S_j} \phi_j\d\al$. One has, denoting $\De_j \eqdef \max_{x \in S_J} \norm{x_j-x}$, 
	\eq{
		\Wass_1( \hat{\al}_N, \al) \leq \max_{1 \leq j \leq N} \De_j.
	}
\end{lem}

\begin{proof}
We define $\pi \in \Mm_+^1(\Om^2)$, a transport plan coupling marginals $\al$ and $\hat{\al}_N$, by imposing for all $f \in \Cc(\Om^2)$, 
\eq{
	\int_{\Om^2} f \d\pi = \sum_{j=1}^N \int_{S_j} f(x,x_j)\phi_j(x) \d\al(x).
}
$\pi$ indeed is a transport plan, since for all $g \in \Cc(\Om)$, 
\begin{align*}
& \int_{\Om^2} g(x) \d\pi(x,y) = \sum_{j=1}^N \int_{S_j} g(x) \phi_j(x) \d\al(x)  \\
 &= \sum_{j=1}^N \int_{\Omega} g(x) \phi_j(x) \d\al(x) \\
 &= \int_{\Omega} g(x) \left( \sum_{j=1}^N \phi_j(x) \right) \d\al(x)
 = \int_{\Omega} g \d\al.
\end{align*}
Also,
\begin{align*}
&\int_{\Om^2} g(y) \d\pi(x,y) = \sum_{j=1}^N \int_{S_j} g(x_j) \phi_j \d\al \\
 &= \sum_{j=1}^N  g(x_j) \int_{\Omega} \phi_j \d\al 
 = \sum_{j=1}^N  \al_j g(x_j) 
 = \int_\Omega g \d\hat{\al}_N.
\end{align*}
By definition of the Wasserstein-1 distance, 
\begin{align*}
&\Wass(\hat{\al}_N,\al)  \leq \int_{\Om^2} \norm{x-y} \d\pi(x,y) \\
& = \sum_{j=1}^N \int_{S_j} \phi_j(x) \norm{x-x_j} \d\al(x) 
 \leq \sum_{j=1}^N \int_{S_j} \phi_j \De_j \d\al \\
& \leq \left( \sum_{i=1}^N \int_{\Omega} \phi_i \d\al \right) \max_{1 \leq j \leq N} \De_j 
 = \max_{1 \leq j \leq N} \De_j.
\end{align*}
\end{proof}

\section{Proof of Proposition~\ref{propEB}}\label{sec-proof-propEB}

Let us stress that the elementary block $T_f(X)$ defined in~\eqref{eq-elem-block-defn} only depends on the law $\al_X$. In the following, for a measure $\al$ we denote $\Tt_f(\al_X)$ the law of $T_f(X)$. The goal is thus to show that $\Tt_f$ is Lipschitz for the Wasserstein distance. 

For a measure $\al \in \Mm(\Xx)$ (where $\Xx=\RR^q$), the measure $\be=\Tt_\Ff(\al) \in \Mm(\Yy)$ (where $\Yy=\RR^r$) is defined via the identity, for all $g \in \Cc(\Yy)$,
\eq{
	\int_\Yy g(y) \d\be(y) = \int_\Xx g \left( \int_\Xx f(z,x) \d\al(x) \right) \d\al(z). 
}
Let us first remark that an elementary block, when view as operating on measures, can be decomposed using the aforementioned push-forward and integration operator, since 
\eq{
	\Tt_\Ff(\al) = f[\cdot,\al]_\sharp \al.  
}

Using the fact that $\Wass_\Xx$ is a norm,
\begin{align*}
& \Wass_\Xx(\Tt_\Ff(\al), \Tt_f(\be)) \\
& \leq \Wass_\Xx(\Tt_\Ff(\al),f[\cdot,\be]_\sharp \al) + \Wass_\Xx(f[\cdot,\be]_\sharp \al,\Tt_f(\be)) \\
& \leq \norm{ f[\cdot,\al]-f[\cdot,\be] }_{L^1(\al)} + \Lip(f[\cdot,\be]) \Wass_\Xx(\al,\be), 
\end{align*}
where we used the Lipschitzness of the push-forward, Proposition~\ref{prop-lip-pushforward}.
Moreover, for $(z_1,z_2) \in \Xx^2$,
\begin{align*}
& \norm{ f[z_1,\be]-f[z_2,\be] }_2 \leq \norm{ f[z_1,\be]-f[z_2,\be] }_1 \\
& = \norm{ \int_\Xx (f(z_1,\cdot)-f(z_2,\cdot))\d\be }_1 \\
& = \sum_{i=1}^r \lvert \int_\Xx (f_i(z_1,\cdot)-f_i(z_2,\cdot)) \d\be \lvert \\
& \leq \sum_{i=1}^r \int_\Xx \lvert f_i(z_1,\cdot)-f_i(z_2,\cdot) \lvert \d\be \\
& \leq \sum_{i=1}^r \Lip_1(f_{i}) \norm{ z_1 - z_2 }_2 
\leq r \Lip_1(f) \norm{ z_1 - z_2 }_2,  \\
\end{align*}
where we denoted by $\Lip_1(f_i)$ a bound on the Lipschitz contant of the function $f_i(\cdot,x)$ for all $x$. 
Hence $\Lip(f[\cdot,\be]) \leq r \Lip_1(f)$. In addition, Lipschitzness of integration, Proposition~\ref{prop-lip-integration} yields
\eq{
\begin{split}
\Wass_\Xx(\Tt_\Ff(\al), \Tt_f(\be)) \leq & r \Lip_2(f) \Wass_\Xx(\al, \be) \\
& + r \Lip_1(f) \Wass_\Xx(\al,\be) \\
\end{split}
}

\section{Parametric Push Forward}

An ingredient of the proof of the universality Theorem~\ref{thm-meas-valued} is the construction of a noise-reshaping function $H$ which maps a uniform noise to another distribution parametrized by $b$.

\begin{lem}\label{lemma-noise-reshaping}
	There exists a continuous map $(b,u) \in \Si_m \times [0,1]^r \mapsto H(b,u)$ so that the random vector $H(b,U)$ has law $\be \eqdef (1-\epsilon) D_\Yy^*(b) + \epsilon \Uu$, where $U$ has density $\Uu$ (uniformly distributed on $[0,1]^r$).
\end{lem}

\begin{proof}
	Since both the input measure $\Uu$ and the output measure $\be$ have densities and have support on convex set, one can use for map $H(b,\cdot)$ the optimal transport map between these two distributions for the squared Euclidean cost, which is known to be a continuous function, see for instance~\cite{santambrogio2015optimal}[Sec.\ 1.7.6]. It is also possible to define a more direct transport map (which is not in general optimal), known as Dacorogna-Moser transport, see for instance~\cite{santambrogio2015optimal}[Box 4.3]. 
\end{proof}
	
\section{Approximation by Fully Connected Networks for Theorem~\ref{thm-meas-valued}} \label{sec-proof-approx-thm}

We now show how to approximate the continuous functions $f$ and $g$ used in Theorem~\ref{thm-meas-valued} by neural networks. Namely, let us first detail the real-valued case (where two Elementary Blocks are needed), namely prove that: $\forall \epsilon > 0$, there exists integers $N, p_1, p_2$ and weight matrices matrices $A_1 \in \RR^{p_1\times d}$, $A_2 \in \RR^{p_2\times N}$, $C_1 \in \RR^{N\times p_1}$, $C_2 \in \RR^{1\times p_2}$; biases $b_1 \in \RR^{p_1}$, $b_2 \in \RR^{p_2}$: i.e., there exists two neural networks
\begin{itemize}
	\item $g_\theta(x) = C_1 \la (A_1 x + b_1) : \RR^d \to \RR^N$
	\item $f_\xi(x) = C_2 \la (A_2 x + b_2) : \RR^N \to \RR$
\end{itemize}
s.t.
$\forall \al \in \Mm_+^1(\Omega), \left\lvert \Ff(\al) - f_\xi \left( \EE_{X \sim \al} \left( g_\theta(X) \right) \right) \right\rvert < \epsilon$.

Let $\epsilon > 0$. By Theorem~\ref{thm-meas-valued}, $\Ff$ can be approximated arbitrarily close (up to $\frac{\epsilon}{3}$) by a composition of functions of the form $f \left( \EE_{X \sim \al} \left( g(X) \right) \right)$. By triangular inequality, we upper-bound the difference of interest $\left\lvert \Ff(\al) - f_\xi \left( \EE_{X \sim \al} \left( g_\theta(X) \right) \right) \right\rvert$ by a sum of three terms:
\begin{itemize}
	\item $\left\lvert \Ff(\al) - f \left( \EE_{X \sim \al} \left( g(X) \right) \right) \right\rvert$
	\item $\left\lvert f \left( \EE_{X \sim \al} \left( g(X) \right) \right) - f_\xi \left( \EE_{X \sim \al} \left( g(X) \right) \right) \right\rvert$
	\item $\left\lvert f_\xi \left( \EE_{X \sim \al} \left( g(X) \right) \right) - f_\xi \left( \EE_{X \sim \al} \left( g_\theta(X) \right) \right) \right\rvert$
\end{itemize}
and bound each term by $\frac{\epsilon}{3}$, which yields the result.
The bound on the first term directly comes from theorem 1 and yields constant $N$ which depends on $\epsilon$. The bound on the second term is a direct application of the universal approximation theorem~\citep{cybenko1989approximation,leshno1993multilayer} (since $\al$ is a probability measure, input values of $f$ lie in a compact subset of $\RR^N$: $\norm{ \int_\Omega g(x) \mathrm{d}\alpha }_\infty \leq \max_{x\in\Omega} \max_i | g_i (x) |$, hence the theorem~\citep{cybenko1989approximation,leshno1993multilayer} is applicable as long as $\la$ is a nonconstant, bounded and continuous function). Let us focus on the third term.
Uniform continuity of $f_\xi$ yields the existence of $\delta > 0$ s.t. $\left\lvert \left\lvert u-v \right\rvert \right\rvert_1 < \delta$ implies $\left\lvert f_\xi(u)-f_\xi(v) \right\rvert < \frac{\epsilon}{3}$. Let us apply the universal approximation theorem: each component $g_i$ of $g$ can be approximated by a neural network $g_{\theta,i}$ up to $\frac{\delta}{N}$. Therefore:
\eq{
\begin{split}
\norm{ \EE_{X \sim \al} & \left( g(X) - g_\theta(X) \right) }_1 \\
& \leq \EE_{X \sim \al} \norm{ g(X) - g_\theta(X) }_1 \\
& \leq \sum_{i=1}^N \int_\Omega | g_i(x) - g_{\theta,i}(x) | \d\al(x) \\
& \leq N \times \frac{\delta}{N} = \delta
\end{split}
}
since $\al$ is a probability measure. Hence the bound on the third term, with $u=\EE_{X \sim \al} \left( g(X) \right)$ and $v=\EE_{X \sim \al} \left( g_\theta(X) \right)$ in the definition of uniform continuity.
We proceed similarly in the general case, and upper-bound the Wasserstein distance by a sum of four terms by triangular inequality. The same ingredients (namely the universal approximation theorem~\citep{cybenko1989approximation,leshno1993multilayer}, since all functions $f$, $g$ and $h$ have input and output values lying in compact sets; and uniform continuity) enable us to conclude.

\section{Proof of Theorem~\ref{thm-density-tensorization}} \label{sec-proof-approx-tenso}

We denote $\Cc(\Mm_+^1(\Omega))$ the space of functions taking probability measures on a compact set $\Omega$ to $\RR$ which are continuous for the weak-$*$ topology. 
We denote the set of integrals of tensorized polynomials on $\Omega$ as
\eq{
\Aa_\Omega \eqdef 
\left\{
\begin{split}
\Ff:\Mm_+^1&(\Omega) \to \RR, \exists n \in \NN, \exists \phi : \Omega^n \to \RR, \\ 
&\forall \mu \in \Mm_+^1(\Omega), \Ff(\mu) = \int_{\Omega^n} \phi \d\mu^{\otimes n} 
\end{split}
\right\}.
}
The goal is to show that $\Aa_\Omega$ is dense in $\Cc(\Mm_+^1(\Omega))$.

Since $\Omega$ is compact, Banach-Alaoglu theorem shows that $\Mm_+^1(\Omega)$ is weakly-$*$ compact. Therefore, in order to use Stone-Weierstrass theorem, to show the density result, we need to show that $\Aa_\Omega$ is an algebra that separates points, and that, for all probability measure $\al$, $\Aa_\Omega$ contains a function that does not vanish at $\al$. For this last point, taking $n=1$ and $\phi=1$ defines the function $\Ff(\al) = \int_\Omega \d\al = 1$ that does not vanish in $\al$ since it is a probability measure. 
Let us then show that $\Aa_\Omega$ is a subalgebra of $\Cc(\Mm_+^1(\Omega))$:
\begin{itemize}
	\item stability by a scalar follows from the definition of~$\Aa_\Om$;
	\item stability by sum: given $(\Ff_1,\Ff_2) \in \Aa_\Omega^2$ (with associated  functions $(\phi_1,\phi_2)$ of degrees $(n_2,n_2)$), 
		denoting $n \eqdef \max(n_1,n_2)$ and $\phi(x_1,\hdots,x_{n}) \eqdef \phi_1(x_1,\hdots,x_{n_1}) + \phi_2(x_1,\hdots,x_{n_2})$ shows that
		$\Ff_1+\Ff_2 =  \int_{\Omega^n} \phi \d\mu^{\otimes n}$ and hence $\Ff_1+\Ff_2 \in \Aa_\Omega$;
	\item stability by product: similarly as for the sum, 
		denoting this time $n = n_1+n_2$ and introducing $\phi(x_1,\hdots,x_{n}) = \phi_1(x_1,\hdots,x_{n_1})\times\phi_2(x_{n_1+1},\ldots,x_{n_3})$ shows that $\Ff=\Ff_1 \times \Ff_2 \in \Aa_\Omega$, using Fubini's theorem.
\end{itemize}
Lastly, we show the separation of points: if two probability measures $(\al,\be)$ on $\Omega$ are different (not equal almost everywhere), there exists a set $\Omega_0 \subset \Omega$ such that $\al(\Omega_0) \neq \be(\Omega_0)$; taking $n=1$ and $\phi=\mathds{1}_{\Omega_0}$, we obtain a function $\Ff \in\Aa_\Omega$ such that $\Ff(\al)\neq\Ff(\be)$.

\end{document}